\documentclass{article}

\usepackage[bookmarks,colorlinks=black,breaklinks]{hyperref}  
\usepackage{graphicx}
\usepackage{amsmath}
\usepackage{amsfonts,amsthm,amssymb,xspace,verbatim,dsfont}
\usepackage{url,algorithm,algpseudocode}
\usepackage{thmtools}
\usepackage{thm-restate}
\usepackage{cleveref}
\usepackage{tikz}
\usetikzlibrary{decorations.pathreplacing,angles,quotes}
\usetikzlibrary{positioning,plotmarks,calc}
\usepackage{bbm}
\usepackage{todonotes}
\usepackage{subfig}
\usepackage{pgfplots}
\pgfplotsset{compat=1.14}
\usepackage{float,collcell,filecontents}
  
\theoremstyle{plain}
\newtheorem{thm}{Theorem}
\newtheorem{lem}{Lemma}
\newtheorem{corollary}{Corollary}

\theoremstyle{definition}
\newcommand{\note}[1]{\marginpar{\tiny *note in TeX*}}

\newcommand{\expec}[1]{\mathbb{E}\left[#1\right]}
\newcommand{\normal}{\mathcal{N}(0,1)}
\newcommand{\gauss}{\mathcal{N}}

\newcommand{\reals}{\mathbb{R}}

\newcommand{\I}{\mathbf{I}}

\newcommand{\frob}[1]{\left\| {#1} \right\|_\text{F}}

\newcommand{\twonorm}[1]{\left\| {#1} \right\|_2}

\newcommand{\sign}[1]{\operatorname{sign}\left(#1\right)}

\newcommand{\diag}{\operatorname{diag}}
\newcommand{\vecc}{\operatorname{vec}}

\newcommand{\ReLU}{\operatorname{ReLU}}

\newcommand{\distop}[2]{\mathrm{dist}\left(#1,#2\right)}

\DeclareMathOperator*{\argmin}{argmin}

\newcommand{\ordereps}[1]{\mathcal{O}_{\epsilon}\left({#1}\right)}
\newcommand{\rbrak}[1]{\left(#1\right)}

\newcommand{\cbrak}[1]{\left\{#1\right\}}

\newcommand{\frsq}[2]{\frac{#1}{\sqrt{#2}}}

\newcommand{\eps}{\epsilon}

\newcommand*{\MinNumber}{0}%
\newcommand*{\MaxNumber}{1}
\newcommand{\ApplyGradient}[1]{%
	\pgfmathsetmacro{\PercentColor}{100.0*(#1-\MinNumber)/(\MaxNumber-\MinNumber)}
	\hspace{-0.33em}\colorbox{red!\PercentColor!black}{}
}
\newcolumntype{R}{>{\collectcell\ApplyGradient}c<{\endcollectcell}}

\usepackage{mathrsfs}
\usepackage[round]{natbib}

\usepackage{fullpage}

\title{Learning ReLU Networks via Alternating Minimization}

\author{ Gauri Jagatap and Chinmay Hegde\thanks{
		The authors are with the Department of Electrical and Computer Engineering,
		Iowa State University, 
		Ames, IA, USA 50011. Email: \{gauri,chinmay\}@iastate.edu.  }}

\date{}

\begin{document}

\maketitle
\begin{abstract}
  We propose and analyze a new family of algorithms for training neural networks with ReLU activations. Our algorithms are based on the technique of alternating minimization: estimating the activation patterns of each ReLU for all given samples, interleaved with weight updates via a least-squares step. The main focus of our paper are 1-hidden layer networks with $k$ hidden neurons and ReLU activation. We show that under standard distributional assumptions on the $d-$dimensional input data, our algorithm provably recovers the true ``ground truth'' parameters in a linearly convergent fashion. This holds as long as the weights are sufficiently well initialized; furthermore, our method requires only $n=\widetilde{O}(dk^2)$ samples. We also analyze the special case of 1-hidden layer networks with skipped connections, commonly used in ResNet-type architectures, and propose a novel initialization strategy for the same. For ReLU based ResNet type networks, we provide the first linear convergence guarantee with an end-to-end algorithm.  We also extend this framework to deeper networks and empirically demonstrate its convergence to a global minimum. 
\end{abstract}

\section{Introduction}\label{sec:intro}

\paragraph{Motivation}

Deep neural networks have found success in a wide range of machine learning applications. However, despite significant empirical success, a rigorous algorithmic understanding of \emph{training} such networks remains far less well understood. 

Our focus in this paper are on a class of neural networks with rectified linear units (ReLUs) as activation functions. The method of choice to train such networks is the popular (stochastic) gradient descent. ReLU networks are computationally less expensive to train when compared to networks with tanh or sigmoid activations since they generally involve simpler gradient update steps. Due to their utility as well as amenability to analysis, several recent papers have addressed the problem of \emph{provably} showing that gradient descent for ReLU networks succeeds under various assumptions~\cite{tian2017symmetry,zhong,li,geleema18}  
\vspace{-0.3cm}
\paragraph{Our contributions}

In this paper, we depart from the standard approach of gradient descent (GD) for learning ReLU-based neural networks. Instead, we propose a new approach based on the technique of \emph{alternating minimization}. In contrast with gradient-based learning, our algorithm is \emph{parameter-free}: it does not involve any tuning parameters --- such as learning rate, damping factor, dropout ratio, etc --- other than setting the number of training epochs. To our knowledge, the alternating minimization framework presented in our paper is novel.

Additionally, we supplement our learning algorithm with a rigorous analysis. In particular, we prove that under a generative modeling assumption where there is a ``ground truth'' (or teacher) network and if the data samples $x$ are distributed according to a multivariate Gaussian, our algorithm exhibits linear convergence provided there are sufficient number of samples. This means that the parameter estimation error reduces to $\varepsilon$ after $O(\log 1/\varepsilon)$ training epochs. For $d$-dimensional inputs, our approach requires {\color{black}$O(dk^2 \text{poly}(\log d))$} for 1-hidden layer networks. 

For 1-hidden layer networks, our analysis works for both dense architectures~\cite{zhong} as well as (a simple modification of) residual networks (ResNets) with skipped connections~\cite{li}\footnote{ResNets conventionally skip two layers of weights before identity mapping of the previous weights; in this paper we skip only one layer.}. We remark that our rate of convergence matches that of the algorithm provided in~\cite{zhong} for dense networks, and improves upon the standard SGD for ResNets, as analyzed in~\cite{li}. However, in contrast with~\cite{zhong}, this paper offers the following improvements: our algorithm is parameter-free and our analysis is simpler. Moreover, empirically, we are able to show that alternating minimization exhibits better success rates, as compared to standard gradient descent, particularly as the network becomes more complex. 

Additionally, we provide an improved analysis of convergence for training a two-layer ResNet type network through our algorithmic framework. We demonstrate that by picking an \textit{identity} initialization, one can prove convergence to global minimum through alternating minimization. Note that this initialization is way more computationally efficient as compared to adopting the tensor decomposition based method in \cite{zhong}. In general, proving convergence of our alternating minimization-based learning algorithm requires a suitable initialization, which we elaborate on in further detail for both standard and residual architectures. To the best of our knowledge, this is the first such theoretical result for provably training (shallow) ResNets with finitely many samples.

Finally, we provide a range of numerical experiments that support our theoretical analysis. Our experiments show that our proposed algorithm provides comparable, and in some cases improved numerical performance than those provided in~\cite{zhong}, which utilizes gradient descent for training their networks. Meanwhile, we show improved numerical performance for ResNet-type networks by utilizing our initialization strategy.  While our theory is only for 1-hidden layer networks, we also extend the algorithm to deeper ReLU networks, and show that the algorithm works well empirically (i.e., it gives zero training loss for sufficiently many samples). Establishing rigorous guarantees for ReLU networks with depth $\geq 2$ is an interesting direction for future work.

Overall, our work can be viewed as a first step towards a new algorithmic approach (beyond traditional gradient approaches) for training ReLU networks.
 \vspace{-0.3cm}
\paragraph{Techniques}
At a high level, our algorithm is based upon a simple (but key) idea that we call the ``linearization'' trick. For a {given} sample $x$ and a given estimate of the network parameter $W$, define the \emph{state} of the network as the collection of binary variables that indicate whether a given ReLU is active or not. (In the literature, these have been referred to as \emph{signatures}~\cite{laurent2017multilinear}.) Since ReLU networks simulate piecewise linear functions, if we \emph{fix} the state, then the the mapping from $x$ to the label $y$ can be approximated using a \emph{linear} neural network.

We can repeat this local linearization step for all samples. Overall, we obtain a system of linear equations with the weights of the network as the variables, which we can solve using either black-box direct methods (such as LU or Cholesky decompositions) or iterative solvers (such as conjugate gradients) depending on the size of the problem. Once we obtain an improved estimate of the weights, we can update the the state of the network for each sample, and repeat the above procedure by alternating between updating the state variables as well as the weights until convergence. 

Two main conceptual challenges arise here. The first is to rigorously prove that our linearization trick actually converges to a global optimum in the absence of noise. (In fact, proving convergence even to a local optimum has not been shown before, to our knowledge.) To show this, we appeal to recent algorithmic advances for solving \emph{phase retrieval} problems~\cite{netrapalli,CandesL12,CandesLS13,copram,copramext,zhang2016reshaped}. In particular, in the case of Gaussian-distributed data we adapt the proof of ~\cite{netrapalli} to prove linear convergence of the weights to an $\epsilon$-ball around the ground truth teacher network parameters. 

Our above alternating minimization succeeds (in theory) for arbitrary networks. However, our analysis is local in nature, and only succeeds provided the network weights are initialized suitably close to the ground truth. The second conceptual challenge is to actually produce such an initialization. We discuss this in the context of several common architectural assumptions. For dense networks, we can leverage the tensor-based initialization of~\cite{zhong}, while for residual networks, we propose a simple \textit{identity} initialization that can be directly derived by inspecting the network architecture. In both cases, we perform several numerical experiments to demonstrate the benefits of our proposed methods.

\begin{table*}[!t]
	\centering
	\caption{{\sl 
		 Comparison of algorithms and analysis for 1-hidden layer networks. Here, $d,k,n$ denote dimensions of the data and hidden layer, number of samples, respectively. $\ordereps{\cdot}$ hides polylogarithmic dependence on $\frac{1}{\epsilon}$. Alternating Minimization and (Stochastic) Gradient descent are denoted as AM and (S)GD respectively.}} \label{tab:compare}
	\resizebox{0.9\textwidth}{!}{
		\begin{tabular}{|p{1cm}|p{3.5cm}|p{3.3cm}|p{1.9cm}|p{2.6cm}|p{1.7cm}|} 
			\hline
			
			Alg. & Paper & Sample complexity & Convergence rate & Architecture & Parameters\\
			\hline 
			GD &  \cite{zhong} & $\ordereps{d k^2 \cdot\text{poly}(\log d)}$ & $\ordereps{\log \frac{1}{\eps}}$ & Dense, under-parametrized & step-size $\eta$\\
			\hline 
			SGD & \cite{wang2018learning} & $\ordereps{k^2\twonorm{w^*}^2}$ & $\ordereps{\frac{1}{\epsilon}}$ & Dense, ReLU,  under-parametrized 
			& step-size $\eta$\\
			\hline 
			SGD &  \cite{li2018learning} & $\ordereps{\text{poly}(k)}$ & $\ordereps{\frac{1}{\epsilon}}$ & Dense, ReLU, over-parametrized & step-size $\eta$\\
			\hline 
			SGD &  \cite{li} & $\times$ (population loss) & $\ordereps{\frac{1}{\epsilon}}$ & Skipped/ResNet, ReLU & step-size $\eta$\\
			\hline 
			GD &\cite{bartlett2018gradient} & $\times$ (population loss) & $\ordereps{\log \frac{1}{\eps}}$ & Skipped/ResNet, linear & step-size $\eta$\\	
			\hline
			\hline
			{AM} & (this paper) & $\ordereps{d k^2\cdot \text{poly}(\log d)}$ &$\ordereps{\log \frac{1}{\eps}}$ &  Dense, under-parametrized, ReLU  & none \\
			\hline 
			{AM} & (this paper) & $\ordereps{d k^2\cdot \text{poly}(\log d)}$ & $\ordereps{\log \frac{1}{\eps}}$ &  Skipped/ResNet, ReLU  & none\\
			\hline
	\end{tabular}}
\end{table*}
\vspace{-0.3cm}

\paragraph{Comparison with prior work}

Owing to the tremendous success of deep learning approaches in practice~\cite{lecun2015deep,krizhevsky2012imagenet}, numerous theoretical contributions have emerged that study various aspects of deep learning algorithms (such as expressive power, accuracy of learning from an optimization perspective, and generalization). Due to space constraints, our overview of related work will be necessarily incomplete; see Section 2 of~\cite{zhong} for a more comprehensive overview.

Our focus in this paper is on designing provably accurate algorithms for optimizing the squared-error training loss for ReLU networks. Several recent papers have studied this from a theoretical perspective. The works of~\cite{tian2017symmetry,soltanolkotabi2017learning} have proved global convergence of gradient descent for a 0-hidden layer network with a ReLU activation, given a sufficient number of samples. 
The case of 1-hidden layer ReLU networks is more challenging, and it is now known that the landscape of the squared-error loss function is rife with local minima; therefore, algorithmic aspects such as initialization and local/global convergence need to be carefully handled. For random initializations, the recent papers~\cite{tian2017symmetry,brutkus-globerson} provides a symmetry-breaking convergence analysis for 2-layer ReLU networks where the weight vectors of the hidden layer possesses disjoint supports. In contrast, our algorithm succeeds without any such assumptions on the weights of the network.

 In \cite{taylor2016training}, the authors explore a similar deviation from gradient descent, for training neural networks, via a combination of alternating directions method of moments (ADMM) and Bregman iterations. The optimization procedure is fairly involved; the weights, activations, and linear inputs prior to activation, each is updated, in an alternating manner, over all layers. The basis of our paper is much simpler; we only estimate weights and non-linearities arising from the ReLU activation.

On the other hand, the approach in~\cite{geleema18} constructs a special non-convex loss function which does not suffer from local minima, and whose minima correspond to the ground truth parameters of the ReLU networks. However, in contrast our metric of success is directly measured in terms of the Euclidean error between the estimated weights and the ground truth parameters.

In adjacent literature, for 1-hidden layer networks, \cite{wang2018learning} proves convergence of modified SGD to global minimum for the problem of classification of linearly separable data by introducing noisy perturbations to the inputs of ReLU neurons. Similarly, \cite{li2018learning} shows that SGD can successfully learn a two-layer over-parametrized network in the setting of multi-class classification for well-separated distributions. 

The work perhaps most closely related to our approach is that of \cite{zhong}, who provide recovery guarantees for 1-hidden layer networks with a variety of activation functions, including ReLU's. Like us, they provide a two-stage scheme: a suitable initialization, followed by gradient descent. (Such a two-stage approach has been proposed for solving several other non-convex machine learning problems~\cite{netrapalli,arora2015simple}). In contrast, our method is based on alternating minimization, and  requires fewer  parameters (and in practice, does not require fine-tuning of factors such as learning rate). 

In addition, in contrast with~\cite{zhong}, our algorithm (and analysis) for 1-hidden layer networks extend to residual networks with skip connections. Residual networks have shown to give empirical advantage over standard networks and are empirically easier to optimize over~\cite{resnets}. 
An added benefit of optimizing with skipped connections is that constitutes a simpler optimization landscape with fewer local minima~\cite{goldstein}. We leverage this fact, and show that even with a random initial estimate, one is able to constrain the set of possible solutions $\mathbf{I}+ W$ to a smaller subset of the space of optimization variables, hence enabling global convergence guarantees \footnote{While preparing this paper we became aware of the work in \cite{bartlett2018gradient}, that also utilizes an identity initialization to train deep residual networks. However, their analysis only holds for \emph{linear} neural networks trained with gradient descent.}.

%
Finally, we remark that a similar alternating minimization framework has been utilized to study the problem of phase retrieval~\cite{netrapalli,phasecut,copram}, with near-optimal computational and statistical guarantees. Indeed, our algorithm for learning 0-hidden layer networks is a direct extension of that of~\cite{netrapalli}. Our contribution in this paper is to show that the same alternating minimization framework generalizes to networks with 1 or more hidden layers as well, and may pave the way for further algorithmic connections between the two problems.

A comparative description of the performance of our algorithms is presented in Table \ref{tab:compare}.

\paragraph{Paper organization}

The rest of this paper is organized as follows. Section~\ref{sec:background} establishes notation and mathematical basics. Section~\ref{sec:Algorithm} introduces our alternating minimization framework for a 1-hidden layer network, establishes convergence bounds, and includes a sketch of the analysis techniques. Section~\ref{sec:experiments} supports our analysis via several numerical experiments. 
\section{Mathematical model}
\label{sec:background}

\paragraph{Notation}
\label{sec:notation}

Some of the notation used in the entirety of this paper are enlisted below. Scalars and vectors are denoted by small case letters, and matrices are denoted by upper case letters, with elements of both indexed by subscripts. The indicator vector $p$ stores the missing sign information from prior to the ReLU operation. The matrix $\mathbb{P} = \diag(p)$ represents the matrix with the elements of $p$ along its diagonal and zero elsewhere. The symbol `$\circ$' denotes the element-wise Hadamard product. Vectors and matrices with superscript `*' denote the ground truth. Vectorization (or flattening) of a matrix $M$ is represented as $\vecc(M)$, producing a long vector with columns of $M$ stacked one beneath another. Small constants are represented by $\delta$ and large constants by $C$. The non-linear operation of ReLU is denoted by $\sigma(\cdot)$. The indicator function is denoted as $\mathbbm{1}_{\{{Xw^*} \geq 0\}} := \frac{\sign{Xw^*}+\mathbf{1}}{2}$, where each entry: 
\[
\mathbbm{1}_{{\{{Xw^*} \geq 0\}}_i} = \begin{cases}
1, &\quad {x_i^\top w^*} \geq 0\\
0, & \quad {x_i^\top w^*} < 0
\end{cases},
\] 
for all $i\in\{1\dots n\}$. Boldface $\mathbf{1}$ represents a vector of ones. Similarly, $\mathbf{I}$ is the identity matrix.

In each of the ReLU network architectures considered below, we assume that the training data consists of i.i.d. samples $\{(x_1, y_1), \ldots, (x_n, y_n)\}$, where each input $x_i \in \reals^d$ is constructed by sampling independently from a zero-mean $d-$variate Gaussian distribution with variance $1/n$, and the output $y_i$ obeys the generative model:
\hspace{-0.1cm}
$$ y_i = f^* (x_i),$$
\hspace{-0.1cm}
where $f^*$ varies depending on the architecture. We denote the data matrix $X \in \reals^{n\times d}$, with each row representing a $d$-dimensional sample point, with $n$ such samples. Scalar labels corresponding to each sample point are contained in $y \in \reals^n$. If skipped connections are considered, we denote the network weights by $W^*$ and effective weights of the forward model as $\hat{W}^* = W^* + \I$.
 

We devise algorithms that recover the ``ground truth'' parameters of the ReLU network. In the case of neural networks with 1 or more hidden layers, such a recovery can only be performed modulo some permutation of the weight vectors. Therefore, to measure the quality of the algorithm, we define a notion of distance between two weight matrices: 
\[
\distop{W}{W'} := \frob{W - W'_{{\pi}}} = \min_{\pi' \in \Pi} \frob{W - W'_{{\pi'}}},
\] 
{\color{black}where $\pi'$ is any permutation of the column indices of the weight matrix, such that $W_{\pi'} = \overline{W}$ (column permuted version of $W$) and $\Pi$ is a set of all possible permutations of the column indices of $W$}. For notational convenience, we assume without loss of generality that $\pi$ is the identity permutation, in which case we simply can replace the $\distop{\cdot}{\cdot}$ with the more familiar Frobenius norm. 

\paragraph{Multi-layer ReLU network} 
A ``teacher" ReLU network with $L$ layers and scalar output can be characterized as follows:
\begin{align}
f^*(X) := \sigma(\sigma(\dots \sigma(XW^{1*})W^{2*} )\dots)W^{L*}), \label{eq:fwd2} 
\end{align}
where $X \in \mathbb{R}^{n\times d_1}$, and the weights of the teacher network at the $l^{th}$ layer are $W^{l*} \in \mathbb{R}^{d_l\times d_{l+1}}$ with $l\in \{1,\dots L-1\}$ and $W^{L*} = \mathbf{1}_{d_L \times 1}$ \footnote{Since the ReLU is a homogenous function, we can assume that the last layer of weights is the all-ones vector without loss of generality.}.

In this paper we specifically study the 1-hidden layer formulation of this problem, with $d_1=d$ and $d_2 = k$ neurons in the hidden layer, as a forward model:
\begin{align}
f^*(X) = \sum_{q=1}^{k} \ReLU(Xw_q^{*}) := \sigma(XW^{*}) \mathbf{1}_{k\times 1}, \label{eq:fwd1hidden}
\end{align}
where we have the weight matrix corresponding to the first layer $W_1^{*} = W^* := [w^{*}_1 \dots w^{*}_k]\in \reals^{d \times k}$. 
We can \textit{linearize} this forward model by introducing an new vector $p_q^* = \mathbbm{1}_{\{{Xw_q^*} > 0\}} := \frac{\sign{Xw_q^*}+\mathbf{1}}{2}, \forall q$ (and matrix counterpart $\mathbb{P}_q := \diag(p_q)$) which stores the sign information of the inputs to each of the hidden neurons. We can rewrite this forward model as a linear function of the weights $W^*$:
\begin{align*}
f^*(X) &= [\diag(p_1^*)X \dots \diag(p_k^*)X]_{n \times dk} \cdot \vecc(W^{*})_{dk \times 1}, \\
&:= [\mathbb{P}_1^*X \dots \mathbb{P}_k^*X]_{n \times dk} \cdot \vecc(W^{*}),\\
&:= [A_1^* \dots A_k^*]\cdot \vecc(W^{*}) := B^* \cdot \vecc(W^{*}).
\end{align*}
We utilize the alternating minimization framework to estimate $\mathbb{P}_q$'s and $\vecc(W)$.

A special case of this formulation is the single neuron case, for which the forward model is simply $y = \ReLU(Xw^*) = p^*\circ \rbrak{X w^{*}} := A^* \cdot w^{*}$, where $w^* \in \reals^{d \times 1}$ are the neuron weights to be recovered.

The learning problem essentially comprises of recovering the underlying mapping $f$, from outputs $y$, such that $f(X)$ estimates  $y = f^*(X)$. This can be formulated as a minimization of the following form,
\begin{align} \label{eq:lossfn}
\min_{W_1,\dots W_L} \mathscr{L}(W_1,\dots, W_L) = \min_{W_1,\dots W_L} \twonorm{f(X) - y}^2
\end{align}
where $\mathscr{L}(W_1,\dots, W_L)$ is the $\ell_2$-squared loss function.

%


\section{Algorithm and Analysis}
\label{sec:Algorithm}

We now propose our alternating minimization-based framework to learn shallow networks with ReLU activations using a $\ell_2$ loss function. At a high level, our algorithm rests on the following idea: given the knowledge of the correct signs of the input to each ReLU, the forward models depicted above can be \emph{linearized}, i.e., they now represent linear neural networks. Therefore, the weights can be estimated as in the case of any other linear model (e.g., via least-squares).

This immediately motivates an iterative two-phase approach: alternate between (i) a sign estimation (linearization) step, followed by (ii) an update to the estimate of the weights via least-squares estimation. The first phase can be computationally achieved by a single forward pass through the network for each sample, while the second phase can be computationally achieved by any standard (stable) solver for linear systems of equations. These steps are iterated to convergence and we elaborate on how many iterations are required below in our analysis.

The pseudo-code for training a 1-hidden layer network using our above framework is provided in the form of Algorithm \ref{algo:hidden1} respectively. The pseudocode for training 2-hidden layer networks is presented in the Appendix \ref{sec:appendixC} as Algorithm \ref{algo:hidden2}.




\renewcommand{\algorithmicrequire}{\textbf{Input:}}
\renewcommand{\algorithmicensure}{\textbf{Output:}}

\begin{algorithm}[!t]
	\caption{Training 1-hidden layer ReLU network via Alternating Minimization}
	\label{algo:hidden1}
	\begin{algorithmic}[1]
		\Require $X,y,T, k$
		\item[]
		\State \textbf{Initialize} $W^0\:\:s.t.\:\: \distop{W^0}{W^*} \leq \delta_1 \frob{W^*}.$
		\item[]	
		\For{$t = 0,\cdots,T-1$}
		\State ${p}^{t}_q \leftarrow \mathbbm{1}_{\{{Xw^t_q} \geq 0\}}$, \quad $\forall q \in \{1\dots k\}$, 		\item [] 
		\State $B^t \gets [\diag(p^{t}_1)X \dots \diag(p^{t}_k)X]_{n \times dk} $,
		\State	$\vecc{(W)}^{t+1} \gets$  $\arg\min\limits_{\vecc(W)} \twonorm{B^t\cdot \vecc(W) - y}^2$,
		\State $W^{t+1} \gets $ reshape$(\vecc{(W)^{t+1}},[d,k])$.
		\EndFor
		\item[]
		\Ensure $W^{T} \leftarrow W^{t}$.
	\end{algorithmic}
\end{algorithm}
\vspace{-8pt}

\subsection{Algorithmic guarantees} 
\label{sec:proof}

We now provide a theoretical analysis of our proposed algorithms. Our first theorem proves that alternating minimization for training a 1-hidden layer ReLU network exhibits linear convergence to the weights of the teacher network.


\begin{restatable}{thm}{convergencetwo}\label{thm:convergence2}
	Given an initialization $W^0$ satisfying $\distop{W^0}{W^*} \leq \delta_1 \frob{W^*}$, for $0 < \delta_1 < 1$, if we have number of training samples $n > C \cdot d\cdot k^2 \cdot \log k $, then with high probability {($1-e^{-\gamma n}$)} the iterates of Algorithm~\ref{algo:hidden1} satisfy:
	\begin{align} \label{eq:mainconvergence2}
	\distop{W^{t+1}}{W^*} \leq {\rho_0} \distop{W^t}{W^*} .
	\end{align}
	Here, {$\gamma$} is a positive constant and  $0 < \rho_0 < 1$ .
\end{restatable}

\begin{proof}
	
	Let $B^t := [\mathbb{P}_1^tX, \ldots, \mathbb{P}_k^tX]_{\pi} := [A_1, \ldots, A_k] = B$, where  $\mathbb{P}^t_q = \text{diag}(p^t_q)$ and subscript $\pi$ represents a specific permutation of the indices $\{1, \ldots, k\}$. For the sake of this proof, we can assume that $\pi$ is the identity permutation, without loss of generality.
	
	Since the minimization in Line 5 of Algorithm \eqref{algo:hidden1} can be solved exactly, we get:
	\begin{align*}
	\vecc(W)^{t+1} &= (B^\top B)^{-1} B^\top y \\
	&= (B^\top B)^{-1} B^\top B^* \vecc(W^*)\\
	&= (B^\top B)^{-1} B^\top B \vecc(W^*) + \\&\hspace{2cm}(B^\top B)^{-1} B^\top (B^* - B) \vecc(W^*).
	\end{align*}
	Taking the difference between the learned weights and the weights of the teacher network,
	\begin{align*}
	\vecc(W)^{t+1} - \vecc(W^*) &= (B^\top B)^{-1} B^\top (B^* - B) \vecc(W^*). 
	\end{align*}
	Taking the vector $\ell_2$ norm on both sides,
	\begin{align} \nonumber
	&\hspace{-0.9cm}\frob{W^{t+1} - W^{*}}\\ \hspace{0.4cm} &= \twonorm{ (B^\top B)^{-1} B^\top (B^* - B^t)\vecc(W^*)}, \\\nonumber
	&\leq \twonorm{(B^\top B)^{-1}}\twonorm{B^\top}	\twonorm{(B^* - B^t)Xw^*},	 \\\nonumber
	&\leq \frac{1}{\sigma_{min}^2(B)} \cdot \sigma_{max}(B) \cdot  \sum_{q=1}^k\twonorm{E_{{sgn}_q}},\\ &\leq \frac{1}{\sigma_{min}^2(B)} \cdot \frac{\sqrt{k}}{1+\sqrt{C}} \cdot \rho_3 \frob{W^t-W^*} \label{eq:phase_rho},
	\end{align}
	where, $\sigma_{min}(B),\sigma_{max}(B)$ are the minimum and maximum singular values of $B$, respectively; $\twonorm{B^\top}$ is bounded via Corollary \ref{cor:Abound} in Appendix \ref{sec:appendixB} and the error $\sum_q^k \twonorm{E_{{sgn}_q}} =  \sum_q^k \twonorm{(\mathbb{P}^*_q - \mathbb{P}^t_q)Xw^*_q}$, is bounded as follows, 
	\begin{align*}
	E_{{sgn}_{q}} &= (A_q-A_q^*)w_q^* = (\mathbb{P}_q^t-\mathbb{P}_q^*)Xw_q^* \\
	\implies \twonorm{E_{{sgn}_q}}^2 &= \sum_{i=1}^{m} (x_i^\top w_q^*)^2 \cdot \mathbf{1}_{\cbrak{(x_i^\top w_q^*)(x_i^\top w_q^t)<0}} \\
	&\leq \rho_{2,q}^2 \twonorm{w_q^t - w_q^*}^2
	\\ \implies \sum_{q=1}^k\twonorm{E_{{sgn}_q}}^2 &\leq \sum_{q=1}^k \rho_{2,q}^2 \twonorm{w_q^t - w_q^*}\\
	&\leq \max_q{\rbrak{\rho_{2,q}^2}} \sum_{q=1}^k \twonorm{w_q^t - w_q^*}^2
	\\  &= \max_q{\rbrak{\rho_{2,q}^2}} \frob{W_{\pi}^t - W_{\pi}^*}^2
	\end{align*}
	where the final bound is obtained via Corollary \ref{cor:phase_error} in Appendix \ref{sec:appendixB}, which holds with probability greater than $1-\eta$, where $\eta$ is small constant close to $0$ as long as $n > C\cdot d \cdot k^2 \cdot \log k$. Subsequently,
	\begin{align*}
	\rbrak{\sum_{q=1}^k \twonorm{E_{{sgn}_q}}}^2 &\leq  k \sum_{q=1}^k \twonorm{E_{{sgn}_q}}^2\\
	&\leq k\max_q{\rbrak{\rho_{2,q}^2}} \frob{W_{\pi}^t - W_{\pi}^*}^2\\ 
	\implies \sum_{q=1}^k \twonorm{E_{{sgn}_q}} &\leq \sqrt{k}\max_q{\rbrak{\rho_{2,q}}}  \frob{W_{\pi}^t - W_{\pi}^*}\\ &= \rho_{3} \frob{W_{\pi}^t - W_{\pi}^*},
	\end{align*}
	To evaluate the final desired bound in \eqref{eq:phase_rho}, we have, we can ensure that $\frac{\sqrt{k}(1+\frac{1}{\sqrt{C}})}{\sigma_{min}^2(B)} \leq \delta$, such that, $\rho_0 = \delta\cdot \rho_3 < 1$, the following convergence is established:
	\begin{align*}
	\frob{W^{t+1}-W^*} &\leq \delta \cdot \rho_{3} \frob{W^{t}-W^*} \\
	&:= \rho_{0}\frob{W^{t}-W^*}
	\end{align*}
	where $\rho_0<1$. This condition is established via Lemma \ref{lemma:sigmamin} in Appendix \ref{sec:appendixB}.
\end{proof}

\subsection{Techniques for initialization}

To prove convergence of Algorithm \ref{algo:hidden1}, we require that the initial weights $W^0$ are such that they meet the constraints $\distop{W^0}{W^*} \leq \delta_1\frob{W^*}$ for (small enough) constant $\delta_1$.

\subsubsection*{{Dense 1-hidden layer ReLU networks:}}

For the 1-hidden layer case, one can opt to use the tensor initialization method proposed by Zhong et. al. in \cite{zhong}. For completeness, we provide a description of this method in the Appendix \ref{sec:appendixC}.  



\subsubsection*{{1-hidden layer with skipped connections:}}

For residual network with skipped connections, we consider a modification of Equation \ref{eq:fwd1hidden}, of the form:
 \begin{align}
 f(X) &= \sum_{q=1}^{k} \sigma(X(w_q^{*}+e^S_q)) \nonumber \\
 &= \sigma(X(W^{*} + \mathbf{I}_S))W^{2*} := \sigma(X(\hat{W}^{*}))W^{2*}, \label{eq:skipped}
 \end{align}
where $W^{*} := [w^{*}_1 \dots w^{*}_k]\in \reals^{d \times k}$, effective weights of the forward mapping $\hat{W^*} = {W^*} + \mathbf{I}_S$, $W^{2*} = \mathbf{1}_{k\times 1}$, and $\mathbf{I}_S \in \reals^{d \times k}$ is a sub-matrix of $\mathbf{I}_{d\times d}$, with $k \leq d$ out of $d$ columns (known locations) picked, and $e^S_q$ are the columns of $\mathbf{I}_S$. Furthermore, a common assumption in the literature is that $\frob{W^*} \leq \gamma$~\cite{li,hardtma,bartlett2018gradient}. For our algorithm, we require the assumption that $\distop{\hat{W}^0}{\hat{W}^*}   \leq \delta_1 \|{\hat{W}^*}\|_{F}$. 
Now, suppose we initialize $\hat{W}^0 \gets \mathbf{I}_S$, where support $S$ is known. Then,
\begin{align}
\distop{\hat{W}^0}{\hat{W}^*}  &=  \distop{\mathbf{I}_S}{{W}^* + \I_S} \nonumber \\
&=  \|{W}^* + \mathbf{I}_S - \mathbf{I}_S\|_{F} = \|W^*\|_F. \label{eq:skipped_init}
\end{align}
If the underlying weights corresponding to the true mapping are such that $\frob{W^*} \leq \delta_1 \|W^* + \I_S\|_F$, then the requirement for convergence of Algorithm \ref{algo:hidden1} is met. That is, if there exists some $\gamma$, such that $$\frob{W^*} \leq \gamma \leq  \frac{\delta_1\sqrt{k}}{1-\delta_1},$$ then the output of the alternating minimization algorithm $W^T$ will converge to $W^*$, via Theorem \ref{thm:convergence2}. \qed


\subsection{Sample complexity}

Through our algorithms, we are able to establish recovery guarantees, under certain sample complexity constraints. For single hidden layer model, we obtain a sample complexity requirement of $O(d k^2 \log k)$ for global convergence, as long as a certain (refer Theorem \ref{thm:convergence2}) weights initialization condition is met. This is comparable to the results in Theorem D.2. of \cite{zhong}, where the authors derive that $O(d k^2 \text{poly}(\log d))$ samples are required to ensure linear local convergence of vanilla gradient descent to learn the weights of a 1-hidden layer network. The sample complexity bottleneck lies in the initialization stage. As a special case, for a single neuron model, our bound equates to requiring $O(d)$ samples for successfully learning the true single neuron mapping $w^*$, using alternating minimization. This matches the results for learning the weights of a single neuron via gradient descent \cite{soltanolkotabi2017learning}.

This is also the first paper that comments on the sample complexity requirement to successfully learn the weights of (one block) of a 1-hidden layer network with skipped connections, as long as the network is initialized to an {identity} mapping (refer Eqn. \ref{eq:skipped_init}). In this case, the sample complexity requirement is exactly $O(d k^2 \log k)$, as long as $\frob{W^*} < \gamma$.

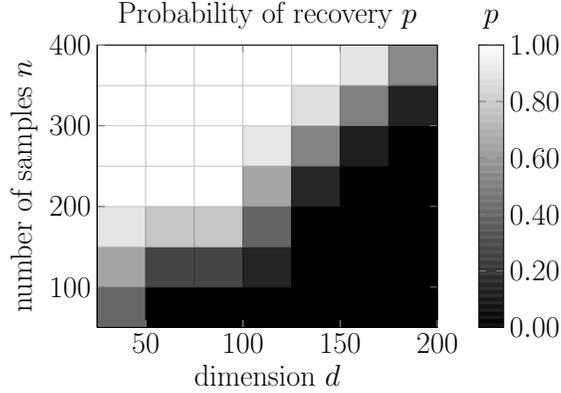
\begin{figure}[!t] 
	\centering
	\begin{tikzpicture}[scale=0.66]
\tikzstyle{every node}=[font=\LARGE]
	\begin{axis}[view={0}{90},
	xlabel=dimension $d$,
	ylabel=number of samples $n$,
	colorbar,
	colormap/blackwhite,
	colorbar style={
		title=$p$,
		yticklabel style={
			/pgf/number format/.cd,
			fixed,
			fixed zerofill
		}
	},
	title=Probability of recovery $p$]
	\addplot3[surf] file {temp.dat};
	\end{axis}
\end{tikzpicture}
\caption{\emph {Phase transition of single neuron model with ReLU activation, for a $d$-dimensional dataset with $n$ samples, and probability of successful recovery $p$.}} \label{fig:phase_trans0}
\end{figure}

\section{Experiments}
\label{sec:experiments}


\paragraph{Single ReLU neuron:}

We select a data matrix $X \in \reals^{n \times d}$ with entries picked from a normal ($\gauss(0,1/n)$) distribution. We construct a (single) ReLU with weights $w^* \in \reals^d$ picked from a Gaussian distribution $\normal$. Our goal is to recover $w^*$, using our proposed approach. In our experiments, we set $w_0 \gets \mathbf{0}$, but we have seen that random initialization with weights $O(1/\sqrt{d})$ (which is typically prescribed in practice) yields identical phase transitions. 

The dimension of the input $d$ is swept in steps of $25$ from $d=50$ to $d=200$, and the number of training samples $n$ is swept from $n=50$ to $n=500$ in steps of 50. The training process was repeated for 100 trials, with each trial corresponding to a different set of random data samples $X$. Recovery is said to be \emph{successful} if the learned weights after $T$ iterations $w^T$, satisfy $\twonorm{w^T-w^*}/\twonorm{w^*} < 0.01$. The probability of recovery is calculated as the ratio of number of trials in which \textit{recovery is successful} and the \textit{total} number of trials. We present the results in terms of a phase transition diagram (Figure~\ref{fig:phase_trans0}). As can be seen from the figure, there appears to be a linear relationship between the dimension $d$ and the number of samples required by alternating minimization for accurate parameter recovery, which is predicted by our theoretical analysis. We also point out that the phase transition is identical to that generated by training this neuron via gradient descent (not shown in this paper); however, a key benefit of our algorithm is that no tuning parameters, such as step size $\eta$, are required to train the network.

\paragraph{One-hidden layer ReLU network:}


\begin{figure}[!t]
	\centering
	\resizebox{\textwidth}{!}{
		\begin{tabular}{cc}
			\begin{tikzpicture}[scale=0.99]
\tikzstyle{every node}=[font=\Large]
\begin{axis}
[width=0.5\textwidth,
xlabel= Number of samples $n$, 
ylabel= Probability of recovery,
grid style = dashed,
grid=both,
legend style=
{at={(1.02,0)}, 
anchor=south west, 
} ,
x label style={at={(axis description cs:0.5,-0.15)},anchor=north},
y label style={at={(axis description cs:-0.12,.5)},anchor=south}
]

\addplot[color=purple, solid,line width=4pt, mark size=2pt, mark=square*] plot coordinates {
(50,0)
(70,0)
(90,0.46)
(100,0.82)
(130,0.92)
(150,0.98)
(170,1)
(190,1)
(210,1)
(230,1)
(250,1)
}; 

\addplot[color=blue, mark size=2pt,line width=4pt, mark=square*] plot coordinates {
(50,0)
(70,0)
(90,0)
(100,0.12)
(130,0.58)
(150,0.76)
(170,0.94)
(190,0.94)
(210,1)
(230,1)
(250,1)
}; 


\addplot[color=red, mark size=4pt,line width=3pt, mark=diamond*] plot coordinates {
(50,0)
(70,0)
(90,0)
(100,0)
(130,0.12)
(150,0.54)
(170,0.86)
(190,0.96)
(210,1)
(230,1)
(250,1)
};

\addplot[color=black, mark size=4pt,line width=3pt, mark=diamond*] plot coordinates {
	(50,0)
	(70,0)
	(90,0)
	(100,0)
	(130,0)
	(150,0.06)
	(170,0.5)
	(190,0.82)
	(210,0.98)
	(230,1)
	(250,1)
};

\addplot[color=magenta,dashed, mark size=3pt,line width=3pt, mark=*] plot coordinates {
(50,0)
(70,0)
(90,0.28)
(100,0.92)
(130,0.94)
(150,0.98)
(170,1)
(190,1)
(210,1)
(230,1)
(250,1)
};

\addplot[color=cyan, dashed,line width=3pt, mark size=3pt, mark=*] plot coordinates {
(50,0)
(70,0)
(90,0)
(100,0)
(130,0.44)
(150,0.86)
(170,0.98)
(190,0.94)
(210,1)
(230,1)
(250,1)
};

\addplot[color=orange, densely dotted, mark size=3pt,line width=3pt, mark=*] plot coordinates {
(50,0)
(70,0)
(90,0)
(110,0)
(130,0)
(150,0.24)
(170,0.82)
(190,0.96)
(210,1)
(230,1)
(250,1)
};

\addplot[color=gray, densely dotted, mark size=3pt,line width=3pt, mark=*] plot coordinates {
	(50,0)
	(70,0)
	(90,0)
	(110,0)
	(130,0)
	(150,0)
	(170,0)
	(190,0.14)
	(210,0.66)
	(230,0.98)
	(250,1)
};

\legend{k=3(AM),k=4(AM), k=5(AM), k=6(AM), k=3(GD), k=4(GD), k=5(GD), k=6(GD)}

\end{axis}
\end{tikzpicture}
&
	
			\begin{tikzpicture}[scale=0.99]
\tikzstyle{every node}=[font=\Large]
\begin{axis}
[width=0.5\textwidth,
xlabel= Number of samples $n$, 
ylabel= Probability of recovery,
grid style = dashed,
grid=both,
legend style=
{at={(1.02,0)}, 
anchor=south west, 
} ,
x label style={at={(axis description cs:0.5,-0.15)},anchor=north},
y label style={at={(axis description cs:-0.12,.5)},anchor=south}
]

\addplot[color=red, mark size=2pt,line width=4pt, mark=square*] plot coordinates {
(100,0.06)
(200,0.58)
(300,0.80)
(400,0.96)
(500,0.94)
(600,0.98)
(700,0.98)
(800,1)
(900,0.98)
(1000,1)
};


\addplot[color=blue,  mark size=2pt,line width=3pt, mark=*] plot coordinates {
(100,0.0)
(200,0.44)
(300,0.82)
(400,1)
(500,0.98)
(600,1)
(700,1)
(800,1)
(900,1)
(1000,1)
};

\addplot[color=darkgray, mark size=2pt,line width=3pt, mark=square*] plot coordinates {
	(100,0.0)
	(200,0.2)
	(300,0.62)
	(400,0.80)
	(500,0.86)
	(600,0.90)
	(700,0.94)
	(800,0.98)
	(900,1)
	(1000,0.98)
};

\addplot[color=orange,dashed, mark size=3pt,line width=4pt, mark=*] plot coordinates {
	(100,0.02)
	(200,0.30)
	(300,0.68)
	(400,0.94)
	(500,0.92)
	(600,0.94)
	(700,0.98)
	(800,0.98)
	(900,1)
	(1000,1)
};

\addplot[color=cyan, dashed, mark size=3pt,line width=3pt, mark=*] plot coordinates {
	(100,0.0)
	(200,0.28)
	(300,0.74)
	(400,0.86)
	(500,0.94)
	(600,1)
	(700,1)
	(800,1)
	(900,0.98)
	(1000,1)
};

\addplot[color=gray, dashed, mark size=3pt,line width=3pt, mark=*] plot coordinates {
	(100,0)
	(200,0.02)
	(300,0.30)
	(400,0.58)
	(500,0.86)
	(600,0.90)
	(700,0.98)
	(800,1)
	(900,1)
	(1000,1)
};

\legend{ k=3(AM),k=4(AM),k=5(AM),k=3(GD),k=4(GD),k=5(GD)}

\end{axis}
\end{tikzpicture}\\
		\Large	(a) & \Large (b)
		\end{tabular} 
	}
	\caption{\emph {Phase transition for 1-hidden layer network, $d=20$, via Alternating Minimization (AM, solid) and Gradient Descent (GD, dashed) for (a) random initialization and (b) good initialization. \label{fig:phase_trans1_2}}}
	
\end{figure}
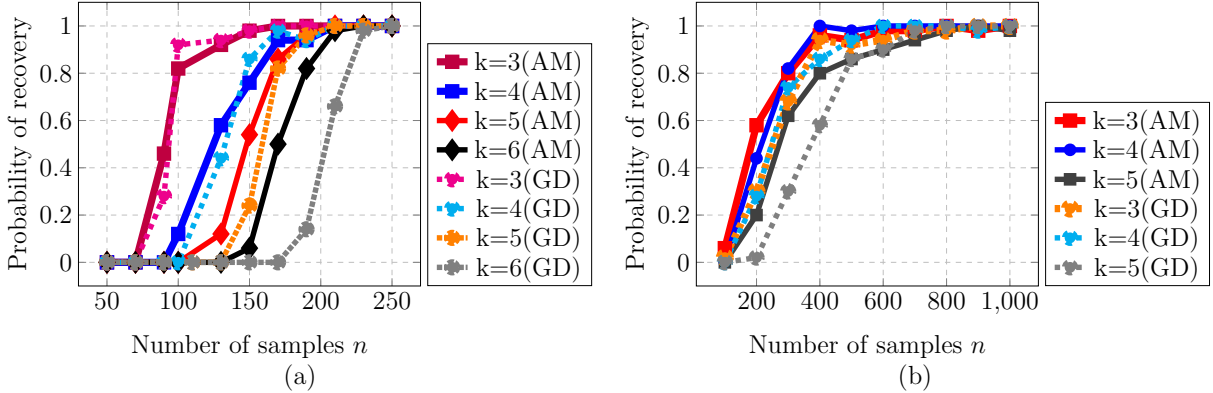

Again, we select the data matrix $X\in \reals^{n \times d}$ with entries picked from a normal $\gauss(0,1/n)$ distribution. Following the setup of \cite{zhong}, we pick $W^*$ such that $W^* = U \Sigma V^\top$, where $U \in \reals^{d\times k}$ and $V \in \reals^{k \times k}$ are obtained from the reduced QR decomposition of a random Gaussian matrix. The diagonal matrix $\Sigma$ is chosen to have entries $1,1+\frac{\kappa-1}{k-1},1+\frac{2(\kappa-1)}{k-1},\dots, \kappa$, with $\kappa=1.8$ (such that $\sigma_1/\sigma_k = \kappa = 1.8$) and varying values of $k$, constituting the forward model in Eqn. \ref{eq:fwd1hidden}.  The weight matrix is initialized\footnote{The procedure in \cite{zhong} can also be used for initialization, however it requires calculation and decomposition of fourth order tensors for ReLU networks. } by (i) perturbing the ground truth such that $\frob{W^0-W^*}\leq 0.9\frob{W^*}$, and (ii) using scaled Gaussian initialization $0.0001 \cdot \gauss(0,\mathbf{I})$. We fix the dimension of the input layer to $d=20$ and vary the dimension of the hidden layer $k=2,3,4,5$. We also sweep the number of training samples $n$ was swept from 100 to 2000 in steps of 100. The training process was repeated for 50 trials. Recovery is said to be successful if $\frob{W^T_\pi-W^*_\pi}/\frob{W^*_\pi} < 0.01$. 

We plot corresponding phase transitions in Figure \ref{fig:phase_trans1_2}, where the recovery performance is measured for $50$ different instantiations (Monte Carlo trials) of the same set of input matrices, with two different initialization schemes. We also compared the performance of our algorithm with that in which the loss function in \eqref{eq:lossfn} is solved using \textit{gradient descent} with appropriately chosen constant step size $\eta$. It may be noted that even though this gives comparatively worse performance than when the algorithm is initialized well (perturbed ground truth), we observe that a random initialization (in a $\delta$-ball around the origin) is capable of recovering weights $W^*$ with sufficiently many samples, with \textit{both} alternating minimization and gradient. 

The phase transition diagrams in Figure~\ref{fig:phase_trans1_2} suggests, that  empirically, alternating minimization is able to perform comparably, or at times even better than gradient descent. This suggests that as the optimization landscape gets more complicated, alternating minimization might be able to avoid certain local minima and saddle points that gradient descent gets stuck in otherwise. To further contrast this performance, we plot training error for a single trial via both algorithms in Figure \ref{fig:phase_trans2} (a). Note that for both good- and random-initializations, the iteration complexity of our method improves upon that of gradient descent. Note that this does \emph{not} mean that our algorithm is necessarily faster (since each iteration requires a least-squares step); however, it does enable the use of out-of-the-box optimized least-squares solvers and avoids learning rate tuning.

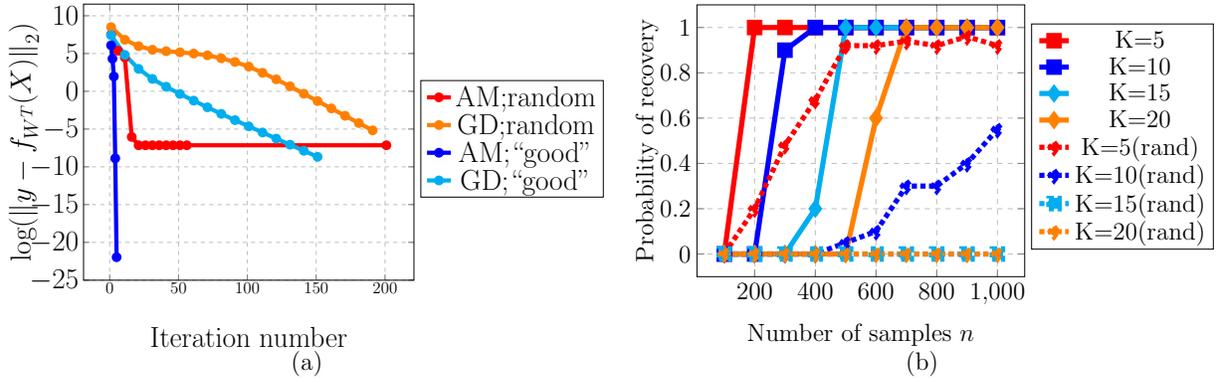
\begin{figure} 

	\centering
	\resizebox{\textwidth}{!}{
		\begin{tabular}{cc}
			\begin{tikzpicture}[scale=0.8]
\tikzstyle{every node}=[font=\huge]
\begin{axis}
[width=0.6\textwidth,
xlabel= Iteration number, 
ylabel= log($\twonorm{y-f_{W^T}(X)}$),
grid style = dashed,
grid=both,
legend style=
{at={(1.02,0.3)}, 
anchor=south west, 
} ,
x label style={at={(axis description cs:0.5,-0.15)},anchor=north},
y label style={at={(axis description cs:-0.12,.5)},anchor=south},
        xticklabel style = {font=\large,yshift=0.5ex}
]

\addplot[color=red, solid,line width=3pt, mark size=2pt, mark=*] plot coordinates {
(1,7.5132)
(6,5.3718)
(11,4.4787)
(16,-6.0394)
(21,-7.1470)
(26,-7.1470)
(31,-7.1470)
(36,-7.1470)
(41,-7.1470)
(46,-7.1470)
(51,-7.1470)
(56,-7.1470)
(201,-7.1470)
};

\addplot[color=orange, solid, mark size=2pt,line width=3pt, mark=*] plot coordinates {
(1,8.4964)
(11,6.8242)
(21,5.9740)
(31,5.5004)
(41,5.3084)
(51,5.1642)
(61,4.9956)
(71,4.7645)
(81,4.3732)
(91,3.8842)
(101,3.2710)
(111,2.4499)
(121,1.5855)
(131,0.6565)
(141,-0.3204)
(151,-1.2944)
(161,-2.2672)
(171,-3.2438)
(181,-4.2038)
(191,-5.1706)
};



\addplot[color=blue, solid,line width=3pt, mark size=2pt, mark=*] plot coordinates {
	(1,6.0855)
	(2,4.3160)
	(3,1.9629)
	(4,-8.8721)
	(5,-21.9657)	
}; 

\addplot[color=cyan, solid,line width=3pt, mark size=2pt, mark=*] plot coordinates {
(1,7.4247)
(11,4.8166)
(21,2.9717)
(31,1.6372)
(41,0.6049)
(51,-0.3487)
(61,-1.2364)
(71,-2.1214)
(81,-2.9859)
(91,-3.8201)
(101,-4.6408)
(111,-5.4550)
(121,-6.2651)
(131,-7.0713)
(141,-7.8742)
(151,-8.6751)
};

\legend{AM;random,GD;random,AM;``good",GD;``good"}

\end{axis}
\end{tikzpicture}	& 
\begin{tikzpicture}[scale=0.99]
\tikzstyle{every node}=[font=\Large]
\begin{axis}
[width=0.5\textwidth,
xlabel= Number of samples $n$, 
ylabel= Probability of recovery,
grid style = dashed,
grid=both,
legend style=
{at={(1.02,0.1)}, 
anchor=south west, 
} ,
x label style={at={(axis description cs:0.5,-0.15)},anchor=north},
y label style={at={(axis description cs:-0.1,.5)},anchor=south}
]

\addplot[color=red, solid,line width=3pt, mark size=3pt, mark=square*] plot coordinates {
	(100,0)
	(200,1)
	(300,1)
	(400,1)
	(500,1)
	(600,1)
	(700,1)
	(800,1)
	(900,1)
	(1000,1)
};

\addplot[color=blue, mark size=3pt,line width=3pt, mark=square*] plot coordinates {
(100,0)
(200,0)
(300,0.9)
(400,1)
(500,1)
(600,1)
(700,1)
(800,1)
(900,1)
(1000,1)
};

%
\addplot[color=cyan, mark size=3pt,line width=3pt, mark=diamond*] plot coordinates {
(100,0)
(200,0)
(300,0)
(400,0.2)
(500,1)
(600,1)
(700,1)
(800,1)
(900,1)
(1000,1)
};

\addplot[color=orange, mark size=3pt,line width=3pt, mark=diamond*] plot coordinates {
(100,0)
(200,0)
(300,0)
(400,0)
(500,0)
(600, 0.6)
(700,1)
(800,1)
(900,1)
(1000,1)
};

\addplot[color=red, dotted, mark size=3pt,line width=3pt, mark=diamond*] plot coordinates {
	(100,0)
	(200,0.2)
	(300,0.48)
	(400,0.68)
	(500,0.92)
	(600,0.92)
	(700,0.94)
	(800,0.92)
	(900,0.96)
	(1000,0.92)
};
\addplot[color=blue, dotted, mark size=3pt,line width=3pt, mark=diamond*] plot coordinates {
	(100,0)
	(200,0)
	(300,0)
	(400,0)
	(500,0.05)
	(600,0.1)
	(700,0.3)
	(800,0.3)
	(900,0.4)
	(1000,0.55)
};

\addplot[color=cyan, dotted, mark size=3pt,line width=3pt, mark=square*] plot coordinates {
	(100,0)
	(200,0)
	(300,0)
	(400,0)
	(500,0)
	(600,0)
	(700,0)
	(800,0)
	(900,0)
	(1000,0)
};

\addplot[color=orange, dotted, mark size=3pt,line width=3pt, mark=diamond*] plot coordinates {
	(100,0)
	(200,0)
	(300,0)
	(400,0)
	(500,0)
	(600,0)
	(700,0)
	(800,0)
	(900,0)
	(1000,0)
};


\legend{K=5,K=10,K=15,K=20,K=5(rand),K=10(rand),K=15(rand),K=20(rand)}

\end{axis}
\end{tikzpicture}\\
		\Large	(a) &\Large (b)
		\end{tabular} 
	}
	\caption{ \textit{(a) Training loss for a fixed trial at $d=20$, $k=5$, and $n=2000$ using alternating minimization and gradient descent under random and ``good" initializations.(b) Phase transition 
		for ResNet-type network, via alternating minimization for $d=20$, with identity (solid) and random (dashed) initializations.  \label{fig:phase_trans2}}}
	
\end{figure}

\paragraph{Residual networks:} We consider a similar 1-hidden layer network as above, but now add \emph{skipped connections} from a random subset of the input layer to the output. Therefore, the ``effective'' weights now become  $\hat{W}^* = W^* + \mathbf{I}_S$ with individual elements of $W^*$ picked from the Gaussian distribution $\gamma\cdot\gauss(0,1/n)$ with $\gamma = 3$ and $\mathbf{I}_S$ being a subset of columns of the identity matrix $\mathbf{I}_{d\times d}$, $S$ representing corresponding indices, and $card(S) = k \leq d$. 

First, we let $d=k=20$, so that the effective weights $\hat{W}^* = \I + W^* \in \reals^{20 \times 20}$. 
We then reduce the number of hidden neurons to $k=5,10,15$, respectively, by dropping or deactivating some pre-fixed neurons. If $S = S_5, S_{10}, S_{15}$ (card$(S_k)=k$) are randomly permuted indices between 1 to 20, 
the weight matrix is initialized simply as $\hat{W}_0 \gets \mathbf{I}_S$.
The number of training samples $n$ was swept from 100 to 1000 in steps of 100. We repeat over 20 trials with each trial being a random instantiation of the data matrix $X$. Recovery is said to be successful if $\|{\hat{W}^T_\pi-\hat{W}^*_\pi}\|_F/\|{\hat{W}^*_\pi}\|_F < 0.01$. 
We observe that our simple identity-based initialization provides a (strictly) better phase transition than a random initialization; see Figure \ref{fig:phase_trans2} (b).

In the supplementary material (Appendix \ref{sec:appendixC}) we demonstrate additional experiments for 2-hidden layer ReLU networks and demonstrate favorable results.

\section{Discussion}

We have provided a new family of algorithms for training ReLU networks with provable guarantees. While our contributions are largely of theoretical interest, several avenues for future work remain: extending our theoretical results to more challenging data distributions, for deeper network architectures; and exploring the impact of our algorithms on larger-scale real world datasets.

\appendix

\section{Key Lemmas} \label{sec:appendixB}

In this section, we state some Lemmas with or without proofs, required for the proof of Theorem \ref{thm:convergence2}.

\begin{lem} \label{lem:Abound}
If $A := \mathbb{P} X$, where $\mathbb{P}$ is a diagonal matrix with indicators 0 and 1 on the diagonal, and the entries of $X$ are from $\gauss(0,1/n)$, then the operation of $\mathbb{P}$ on Gaussian matrix $X$, extracts a sub-matrix of $X$ such that,
\begin{align*}
\twonorm{A^\top} \leq 1 + \frac{1}{\sqrt{C}} \approx 1,
\end{align*}
with probability greater than $1-\frac{\eta}{10}$, if $n > C d \cdot \log{\frac{1}{\eta}}$, where $C$ is large enough and $\eta$ is a constant.

\end{lem}
\begin{proof}
Then spectral norm of the indicator matrix is bounded as $0 \leq \twonorm{\mathbb{P}} \leq 1$. Subsequently,
\begin{align*}
\twonorm{A^\top} = \twonorm{A} &\leq \twonorm{\mathbb{P}}\twonorm{X} \\
&\leq 1\cdot \twonorm{X} \leq 1 + \frac{1}{\sqrt{C}},
\end{align*}
where, $$1-\frac{1}{\sqrt{C}} \leq \sigma_{min}(X) \leq \twonorm{X} \leq \sigma_{max}(X) \leq 1+\frac{1}{\sqrt{C}}$$ is bounded using standard results from random matrix theory \cite{Vershynin10}, with probability greater than $1-\frac{\eta}{10}$, if $n > C d \cdot \log{\frac{1}{\eta}}$, where $C$ is a large constant and $\eta$ is a constant.
\end{proof}

\begin{corollary} \label{cor:Abound}
	For $B := [A_1 \dots A_q \dots A_k ] := [\mathbb{P}_1 X\dots \mathbb{P}_q X \dots \mathbb{P}_k X]$, where $\mathbb{P}_q$'s are diagonal matrices with indicators 0 and 1 on the diagonal, and the entries of $X$ are from $\gauss(0,1/{n})$, then the operation of each $\mathbb{P}_k$ on Gaussian matrix $X$, extracts a sub-matrix of $X$ such that:
	\begin{align*}
	\twonorm{B^\top} = \sigma_{max}{(B)} &= \twonorm{\mathbb{P}_1 X\:\: \mathbb{P}_2 X \dots \mathbb{P}_k X}\\
	&\leq \twonorm{ X \dots (k\: times) \dots  X} \\&\leq {\sqrt{k}}\rbrak{1+\frac{1}{\sqrt{C}}} 
	\end{align*} 	
	as long as $n > C\cdot d \cdot \log \frac{1}{\eta}$, with probability greater than $1-\frac{\eta}{10}$. 
\end{corollary}

\begin{lem} \label{lemma:sigmamin}
The Hessian of the square of the cost function in \eqref{eq:lossfn}, for a 1-hidden layer ($L=2$) network is $\nabla^2_{W} \mathscr{L}(W) = B^\top B $, where $\mathscr{L}(W)$ is defined as $\mathscr{L}(W) = \twonorm{B\cdot \vecc{(W)} - y}= \twonorm{\sum\limits_{q=1}^k \sigma(Xw_q) -y}$ (and elements of $X\sim \gauss{(0,1/n)}$), is bounded as:
$$ \Omega(1/(\kappa^2 \lambda)) \mathbf{I} \preceq \nabla_W^2 \mathscr{L}(W) \preceq O(k) \mathbf{I}$$
where the singular values of $W^*$, and the condition numbers $\kappa$ and $\lambda$ are defined as $\sigma_1 \geq \dots \geq \sigma_k$, and $\kappa = \sigma_1/\sigma_k$ and $\lambda = \left(\prod\limits_{q=1}^k\sigma_q\right)  /\sigma_k^k$, as long as the following hold: $n \geq d\cdot k^2 \text{poly}(\log d,t, \lambda, \kappa)$, $\twonorm{W - W^*} \lessapprox \frac{1}{k^2 \kappa^5 \lambda^2} \twonorm{W^*}$, with probability at least $1-d^{-\Omega(t)}$. This result can be re-interpreted as,
$$ \Omega(1/(\kappa^2 \lambda))\leq \sigma_{min}(B^\top B) \leq \sigma_{max}(B^\top B) \leq O(k) $$
and subsequently,
$$  \Omega(1/(\kappa^2 \lambda))\leq \sigma_{min}^2(B) \leq \sigma_{max}^2(B) \leq O(k). $$

Hence $\sigma_{min}^2(B) \geq c \cdot (\kappa^2 \lambda)^{-1} \equiv \delta$, for small enough $c$. 
\end{lem}

This Lemma has been adapted from Theorem 4.2 of \cite{zhong} for $\phi(z) = \max(z,0)$, by substituting the following values in their result: $p=0$, $v_{min}^*,v_{max}^* = 1$, $\rho = 0.091$.

Subsequently, $\rho_3$ (which depends monotonically on the initial distance $\delta_0$) can be controlled according to the value of $\delta$, such that $\rho_0 = \delta \cdot \rho_3 < 1$, in proof of Theorem \ref{thm:convergence2}. Hence a poor $\delta$ can be compensated by initializing closer to ground truth.

\textbf{Note:} The upper bound $\sigma_{max}^2 \leq O(k)$, matches with our result from Corollary \ref{cor:Abound} which states that $\sigma_{max}^2(B) \leq k\rbrak{1+\frac{1}{\sqrt{C}}}$.

\begin{lem} \label{lem:phase_error}
	As long as the initial estimate $w^0_q$ is a small distance away from the true weights $w^*_q$, $\distop{w^0_q}{w^*_q} \leq \delta_0 \twonorm{w^*_q}$,
	and subsequently,
	$\distop{w^t_q}{w^*_q} \leq \delta_0 \twonorm{w^*_q}$, where $w^t_q$ is the $t^{th}$ estimate of weight vector $w_q$, then the following bound holds,
	\begin{align*}
	 \sum_{i=1}^{m} \rbrak{x_i^\top w^*_q}^2\cdot \mathbf{1}_{\cbrak{ (x_i^\top w^t_q)(x_i^\top w^*_q)\leq 0}} \leq  \rho_{2,q}^2 \twonorm{w^t_q - w^*_q}^2,
	\end{align*}
	with probability greater than $1-\eta$, where $\eta$ is a small positive constant close to $0$, as long as $ n > C \cdot d$ and $\rho_2^2 \approx 0.14$, if $\delta_0 = \frac{1}{10}$ and elements of $x_i,  x_{ij}\sim \gauss(0,1/n)$. 
\end{lem}
 This lemma has been adapted from Lemma 3 of  \cite{zhang2016reshaped}. Note that $\rho_2$ is a monotonically increasing function of $\delta_0$.

\begin{corollary} \label{cor:phase_error}
	As long as the initial estimate $W^0$ is a small distance away from the true weights $W^*$, $\distop{W^0}{W^*} \leq \delta_1 \frob{W^*}$,
	and subsequently,
	$\distop{w^t_q}{w^*_q} \leq \delta_{0} \twonorm{w^*_q}$, where $w^t_q$ is the $t^{th}$ estimate of weight vector $w_q$, for all $q\in\{1,\dots k\}$ then the following bound holds,
	\begin{align*}
 \twonorm{E_{{sgn}_q}}^2 &= \sum_{i=1}^{n} \rbrak{x_i^\top w^*_q}^2\cdot \mathbbm{1}_{\cbrak{ (x_i^\top w^t_q)(x_i^\top w^*_q)\leq 0}} \\
 &\leq  \frac{\rho_2^2}{k} \twonorm{w^t_q - w^*_q}^2,
	\end{align*}
	with probability greater than $1-\eta'$, where $\eta'$ is a small constant close to $0$, as long as $ n > C \cdot d k^2$ and $\rho^2_{2,q} \approx 0.14$, if $\delta_{0} = \frac{1}{10}$ and elements of $X$, $x_{ij}\sim \gauss(0,1/{n})$.
	
	Subsequently, for $k$ such neurons $w_1,w_2,\dots w_k$, we take a union bound over this probability, such that,	
	\begin{align*}
	\rbrak{\sum_{q=1}^k \twonorm{E_{{sgn}_q}}}^2 &\leq  k \sum_{q=1}^k \twonorm{E_{{sgn}_q}}^2 \\ &\leq k\max_q{\rbrak{\rho_{2,q}^2}} \frob{W_{\pi}^t - W_{\pi}^*}^2\\ &\leq {\rho_{3}^2} \frob{W_{\pi}^t - W_{\pi}^*}^2 
	\end{align*}
	where $\rho_3^2 < 0.14$, with probability greater than $1-\eta$, to yield sample complexity $n > C \cdot d\cdot k^2 \cdot \log k$, where $\eta$ is a small constant close to 0. 
\end{corollary}
\begin{proof}
 This lemma has been adapted from Lemma 3 of  \cite{zhang2016reshaped} by modifying equation (58) and subsequently equation (63) as follows:
 \begin{align} \label{eq:single_rho}
 \twonorm{E_{{sgn}_q}}^2 \leq (0.13 + c\epsilon)\twonorm{w^t_q - w^*_q}^2
 \end{align}
 with probability greater than $1-(1+\frac{2}{\epsilon})^d e^{-cn\epsilon^2}$ where $c$ is a small constant, and the factor $c_o = 0.13$ can be reduced by reducing the initial distance factor $\delta_0$ (i.e., $c_o \propto \delta_0 $, refer Proof of Lemma 3 of \cite{zhang2016reshaped} for further details). 
 
To deconstruct $\delta_1$, we have 
 \begin{align*}
  \distop{W^{t+1}}{W^*} &\leq \delta_1 \distop{W^{t}}{W^*}\\
  \sum_{q=1}^k \distop{w_q^{t+1}}{w_q^*}^2 &\leq \delta_1^2 \sum_{q=1}^k \distop{w_q^{t}}{w_q^*}^2
 \end{align*}
We effectively require for each neuron, that 
 \begin{align*}
 \distop{w_q^{t+1}}{w_q^*}^2 &\lessapprox \delta_1^2 \distop{w_q^{t}}{w_q^*}^2
\end{align*}
Therefore $\delta_1 \approx \delta_0$.
 We need to ensure that the value of $\rho_3 < 1$ in the Equation \eqref{eq:phase_rho}. For this, we need to ensure that \textit{each} of $\rho_{2,q}^2 < \frac{1}{k}$. Hence we are required to evaluate Equation \eqref{eq:single_rho}, such that:
 \begin{align} \label{eq:pre_union}
\twonorm{E_{{sgn}_q}}^2 \leq \frac{1}{{k}}(0.13 + c\epsilon)\twonorm{w^t_q - w^*_q}^2
 \end{align}
for $q^{th}$ neuron. The probability of this event is $1-(1+\frac{2}{\epsilon})^d e^{-cn\frac{\epsilon^2}{k^2}}$ where $c$ is a small constant. We further require this condition to hold for all $k$ neurons. Hence we take a union bound, such that \ref{eq:pre_union} holds with probability  $1-k(1+\frac{2}{\epsilon})^d e^{-cn\frac{\epsilon^2}{k^2}}$ for all $k$ neurons. To evaluate sample complexity, 
\begin{gather*}
k\rbrak{1+\frac{2}{\epsilon}}^d e^{-cn\frac{\epsilon^2}{k^2}} < \eta \\
	\implies (\log k)  \cdot d \cdot  \log\rbrak{1+\frac{2}{\epsilon}} - cn\frac{\epsilon^2}{k^2}< \log \eta \\
	\implies n > C \rbrak{k^2 \log k \cdot d \cdot \frac{1}{\epsilon^{2}} \log \frac{1}{\epsilon} + \frac{1}{\epsilon^{2}} \cdot k^2 \cdot \log \frac{1}{\eta}},
\end{gather*}
where $C$ is a constant large enough.
\end{proof}

\section{Supplementary algorithms and experiments} \label{sec:appendixC}

\begin{algorithm}[!t]
	\caption{Training 2-hidden layer ReLU network via Alternating Minimization}
	\label{algo:hidden2}
	\begin{algorithmic}[1]
		\Require $X,y,T, k, k_o$
		\item[]
		\item \textbf{Initialize:} $W_1^0,W_2^0$ s.t. :
		\item[] $\distop{W_1^0}{W_1^*} \leq \delta_1 \frob{W_1^*}$,
		\item[] $\distop{W_2^0}{W_2^*} \leq \delta_2 \frob{W_2^*}$. 
		\item[]	
		\For{$t = 0,\cdots,T-1$}
		\State	${p}^{1,t}_{q} \leftarrow \mathbbm{1}_{\{X w^{1,t}_{q_1}>0\}} $, \quad $\forall q \in [d_2]$,
		\State	${p}^{2,t}_{r} \leftarrow \mathbbm{1}_{\{\sigma{(XW^{1,t})} w^{2,t}_r>0\}}$, \quad $\forall r \in [d_3]$,
		\item []
		\State $C_q^t \gets \sum_{r=1}^{k_o} w_{rq}^{2,t}$ diag$(p_r^{2,t}) \cdot$ diag$(p_q^{1,t}) \cdot X$, \quad $\forall q$,
		\State $\mathbf{C}^t \gets [C_1^t \dots C_{d_2}^t]$,
		\State $\vecc(W^1)^{t+1} \gets \argmin\limits_{\vecc(W^{1})}\twonorm{\mathbf{C}^t\cdot \vecc(W^{1})-y}$,
		\State $W^{1,t+1} \gets \text{reshape}(\vecc(W^{1,t+1}),[d_1,d_2])$,
		\item[]
		\State $B_r^t \gets \mathbb{P}_r^{2,t}\cdot [\mathbb{P}^{1,t}_1 X w_1^{1,t+1} \dots \mathbb{P}^{1,t}_{d_2} X w_{d_2}^{1,t+1}]$,  $\forall r$,
		\State $\mathbf{B}^t \gets [B_1^t \dots B_{d_3}^t]$,
		\State $\vecc(W^2)^{t+1} \gets \argmin \limits_{vec(W^{2})}\twonorm{\mathbf{B}^t\cdot vec(W^{2})-y}$,
		\State $W^{2,t+1} \gets \text{reshape}(\vecc(W^{2,t+1}),[d_2,d_3])$.
		\EndFor
		\item[]
		\Ensure $W^{1,T} \leftarrow W^{1,t}$, $W^{2,T} \leftarrow W^{2,t}$.
	\end{algorithmic}
\end{algorithm}

\subsection{Training 2-hidden layer network with ReLU activation}

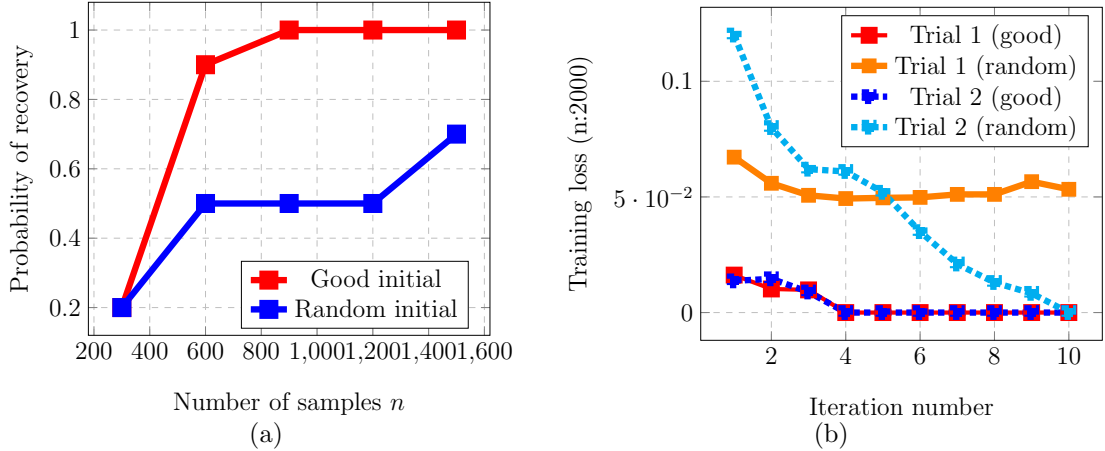
\begin{figure}[!h]
	\centering 
	\begin{tabular}{cc}
		\begin{tikzpicture}[scale=0.8]
\tikzstyle{every node}=[font=\large]
\begin{axis}
[width=0.5\textwidth,
xlabel= Number of samples $n$, 
ylabel= Probability of recovery,
grid style = dashed,
grid=both,
legend style=
{at={(0.38,0.02)}, 
anchor=south west, 
} ,
x label style={at={(axis description cs:0.5,-0.15)},anchor=north},
y label style={at={(axis description cs:-0.12,.5)},anchor=south}
]

\addplot[color=red, solid,line width=3pt, mark size=3pt, mark=square*] plot coordinates {
	(300,0.2)
	(600,0.9)
	(900,1)
	(1200,1)
	(1500,1)

};

\addplot[color=blue, solid,line width=3pt, mark size=3pt, mark=square*] plot coordinates {
	(300,0.2)
	(600,0.5)
	(900,0.5)
	(1200,0.5)
	(1500,0.7)
	
};

%
%
%
%

%

\legend{Good initial,Random initial}

\end{axis}
\end{tikzpicture} &\begin{tikzpicture}[scale=0.8]
\tikzstyle{every node}=[font=\large]
\begin{axis}
[width=0.5\textwidth,
xlabel= Iteration number, 
ylabel= Training loss (n:2000),
grid style = dashed,
grid=both,
legend style=
{at={(0.35,0.58)}, 
anchor=south west, 
} ,
x label style={at={(axis description cs:0.5,-0.15)},anchor=north},
y label style={at={(axis description cs:-0.25,.5)},anchor=south}
]

\addplot[color=red,line width=2pt, mark size=3pt, mark=square*] 
plot coordinates {
(1,0.0162)
(2,0.0102)
(3,0.0098)
(4,0)
(5,0)
(6,0)
(7,0)
(8,0)
(9,0)
(10,0)
};

\addplot[color=orange, mark size=2pt,line width=3pt, mark=square*] plot coordinates {
(1,0.0672)
(2,0.0559)
(3,0.0507)
(4,0.0493)
(5,0.0496)
(6,0.0498)
(7,0.0511)
(8,0.0511)
(9,0.0566)
(10,0.0533)
};

\addplot[color=blue, dotted, mark size=2pt,line width=3pt, mark=square*] plot coordinates {
(1,0.0137)
(2,0.0146)
(3,0.0089)
(4,0)
(5,0)
(6,0)
(7,0)
(8,0)
(9,0)
(10,0)
};

\addplot[color=cyan, dotted, mark size=2pt,line width=3pt, mark=square*] plot coordinates {
	(1,0.12)
	(2,0.08)
	(3,0.062)
	(4,0.061)
	(5,0.052)
	(6,0.035)
	(7,0.021)
	(8,0.013)
	(9,0.0083)
	(10,0)
};


\legend{Trial 1 (good),Trial 1 (random), Trial 2 (good), Trial 2 (random)}

\end{axis}
\end{tikzpicture}\\
		(a) & (b)	
	\end{tabular}
	\caption{\textit{(a) Sample complexity and (b) training loss for $n=2,000$ for 2-hidden layer network (\ref{eq:fwd2}) with $d=20,k=3,k_o=2$ and ReLU activation.} } \label{fig:layer3}
\end{figure}

The alternating minimization framework can similarly be extended for learning a 2-hidden layer network (Algorithm~\ref{algo:hidden2})(with $\mathbb{P}^{l,t}_q :=  \diag(p^{l,t}_q)$).  
For experimental validation, entries of the data matrix $X\in \reals^{n\times d}$ are picked from $\gauss(0,1/n)$ distribution. A three layer network is considered (Eqn. \ref{eq:fwd2}), and is assigned weights $W^{1*}$ and $W^{2*}$ with individual elements of both matrices picked from Gaussian distribution $\normal$. For the sake of experimental analysis, the weight matrices are initialized as $[W_0^1, W_0^2] \gets [W^{1*} + E^1, W^{2*}+ E^2]$, where $\twonorm{E^1} = 0.2\twonorm{W^{1*}}$ and $\twonorm{E^2} = 0.2\twonorm{W^{2*}}$. We fix the dimension of the input layer to $d=20$, the dimension of the first hidden layer to $k=3$ and the dimension of the second hidden layer to $k_o=2$. The number of training samples was set to $n=300,600,900,1,200,1,500$ (Figure \ref{fig:layer3}(a)) for 10 trials, and $2,000$ for 2 trials (Figure \ref{fig:layer3}(b)). 

\subsection{Initialization for 1-hidden layer networks}


Zhong et. al. in \cite{zhong} develop tensor initialization method (one of the approaches for \textsc{InitializeTwoLayer} subroutine in Algorithm \ref{algo:hidden1}) for one-hidden layer networks for ReLU and other activations. The basis of their scheme is to obtain an orthonormal set of vectors $W^0:=\{w_1^0,w_2^0 \dots w_k^0\}$, that have the same span as the true weight vectors $W^*:=\{w_1^*,w_2^* \dots w_k^*\}$. 

To do this, they evaluate an empirical estimator $\hat{P} = \sum_{i=1}^m y_i(x_i x_i^\top - \mathbf{I})$, such that $P = \expec{\hat{P}} = \sum_{q=1}^K \rbrak{\gamma_2(w_q^*) - \gamma_0(w_q^*)}w_q^*{w_{q}^*}^\top$ (note that in this paper, we assume that the second layer consists of weights $v_q^* = 1$, for $q=\{1,2,\dots k\}$). To estimate the initial span $\{w_1^0,w_2^0 \dots w_k^0\}$, the authors contruct two matrices $\hat{P_1} = C\mathbf{I}+\hat{P}$ and $\hat{P_2} = C\mathbf{I} - \hat{P}$ (where $C > 2\twonorm{P}$) and evaluate the top-k eigenvectors and corresponding eigenvalues (in terms of absolute value), each, of $\hat{P_1}$ and $\hat{P_2}$ respectively. They merge the top $k_1$ and $k_2$ eigenvalues of $\hat{P_1}$ and $\hat{P_2}$ respectively, such that $k_1+k_2 = k$. This is done by ordering all eigenvectors, of $\hat{P_1}$ and $\hat{P_2}$ in descending order and picking eigenvectors corresponding to the top-k eigenvalues from the combined pool of eigenvalues. This is followed by an orthogonalization step, such that all $k_1+k_2$ eigenvectors extracted are orthogonal to each other. If the singular values of $W^*$ are $\sigma_1 > \dots > \sigma_k$, then, condition number is $\kappa = \sigma_1/\sigma_k$. 
Then, the authors claim that the procedure described above gives a good initialization:
\begin{align*}
\frob{W^0 - W^*} \leq \epsilon_0 \cdot \text{poly} (k,\kappa)\frob{W^*},
\end{align*}
with high probability if number of samples $n > \epsilon_0^{-2} d\cdot \text{poly} (k,\kappa,\log d)$ (refer Theorem 5.6 of \cite{zhong}).

\section*{Acknowledgements}
This work was supported in part by NSF grants CCF-1566281, CAREER CCF-1750920, CCF-1815101, and a faculty fellowship from the Black and Veatch Foundation. The authors would like to thank Jason Lee for spotting an error in a previous version of our manuscript, and Praneeth Narayanamurthy, Ameya Joshi and Thanh Nguyen for useful discussion and feedback.

\bibliographystyle{abbrvnat}
\bibliography{biblio_nips_relu}

\end{document}